\documentclass{article}
\pdfoutput=1
\PassOptionsToPackage{numbers}{natbib}

\usepackage[preprint]{neurips_2024}

\usepackage[utf8]{inputenc} 
\usepackage[T1]{fontenc}    
\usepackage{url}            
\usepackage{booktabs}       
\usepackage{amsfonts}       
\usepackage{nicefrac}       
\usepackage{microtype}      
\usepackage{xcolor}         
\usepackage{multirow}
\usepackage{subcaption}
\usepackage{tikz}
\usetikzlibrary{positioning, calc}
\usepackage[
    colorlinks=false,      
    citebordercolor=green, 
    urlbordercolor=cyan,   
]{hyperref}
\usepackage{amsmath}
\usepackage{amsthm}
\usepackage{algorithm}
\usepackage{algpseudocode}
\usepackage{graphicx}
\usepackage{svg}
\usepackage{mathtools}

\usepackage{thmtools}
\usepackage{thm-restate}

\newcommand{\bRd}{\mathbb{R}^d}

\newcommand{\norm}[1]{\left\lVert#1\right\rVert}

\newcommand{\bx}{\mathbf{x}}
\newcommand{\by}{\mathbf{y}}
\newcommand{\bz}{\mathbf{z}}
\newcommand{\bq}{\mathbf{q}}

\newtheorem*{theorem*}{Theorem}
\newtheorem{lemma}{Lemma}[section]
\newtheorem{corollary}{Corollary}[section]

\hypersetup{
  pdftitle={Reversible Deep Equilibrium Models},
  pdfauthor={Sam McCallum, Kamran Arora, James Foster}
}

\title{Reversible Deep Equilibrium Models}

\author{Sam McCallum \qquad Kamran Arora \qquad James Foster \\[2pt]
	Department of Mathematical Sciences, University of Bath \\
	\texttt{\{sm2942, ka679, jmf68\}@\hspace{0.8pt}bath.ac.uk}}

\begin{document}

\maketitle

\begin{abstract}
Deep Equilibrium Models (DEQs) are an interesting class of implicit model where the model output is implicitly defined as the fixed point of a learned function. These models have been shown to outperform explicit (fixed-depth) models in large-scale tasks by trading many deep layers for a single layer that is iterated many times. However, gradient calculation through DEQs is approximate. This often leads to unstable training dynamics and requires regularisation or many function evaluations to fix. Here, we introduce Reversible Deep Equilibrium Models (RevDEQs) that allow for exact gradient calculation, no regularisation and far fewer function evaluations than DEQs. We show that RevDEQs significantly improve performance on language modelling and image classification tasks against comparable implicit and explicit models.
\end{abstract}

\section{Introduction}
Explicit models, where the forward model evaluation is explicitly defined and fixed, are the dominant modelling paradigm in machine learning. However, recent work has shown that it is possible to define the forward model implicitly, where a numerical algorithm is evaluated at runtime to determine the model output. This is instantiated by Neural Differential Equations \cite{chen2018neural, gholaminejad2019anode, kidger2020neural, li2020scalable, kidger2021neural, issa2024non, kidger2022neural}, Optimisation Networks \cite{amos2017optnet, gould2021deep} and Deep Equilibrium Models (DEQs) \cite{bai2019deep, el2021implicit}.

Implicit models display interesting properties over their explicit counterparts: compute can be dynamically allocated per data example, memory usage is constant with respect to the effective model depth and oftentimes task performance is improved. Here, we focus on DEQs as the most general and performant class of implicit model with many applications: language \cite{bai2019deep}, classification and segmentation \cite{bai2020multiscale}, generative modelling \cite{bai2022deep, lu2021implicit, pokle2022deep}, inverse problems \cite{gilton2021deep}, implicit representations \cite{huang2021textrm} and graph neural networks \cite{gu2020implicit, liu2022mgnni}.

However, a number of fundamental issues remain with the formulation of DEQs. Backpropagation is performed by computing a numerical solution to the implicit function theorem \cite{bai2019deep}. The gradients obtained are therefore an approximation to the true gradient of the model being trained. This approximation causes training instability and oftentimes training diverges altogether \cite{bai2021stabilizing}. To alleviate this, many function evaluations are required to compute an accurate gradient or strong regularisations are needed to ensure the gradient system is simpler to solve \cite{bai2021stabilizing, winston2020monotone, gabor2024positive}.

Here, we introduce Reversible Deep Equilibrium Models (RevDEQs) that compute exact gradients via an algebraically reversible fixed point solver, while retaining constant memory complexity with respect to the effective model depth. As gradient calculation is exact, RevDEQs require far fewer function evaluations, no regularisation and demonstrate significantly improved performance on large-scale tasks. Code available at \url{https://github.com/sammccallum/reversible-deq}.

\section{Background}
\subsection{Deep Equilibrium Models}
A DEQ defines the model output, $\mathbf{y}$, as the fixed point of a parameterised function \cite{bai2019deep},
\begin{equation}
    \label{eq:deq}
    \begin{aligned}
    &\mathbf{z}^* = f_\theta(\mathbf{z}^*, \mathbf{x}), \\
    &\mathbf{y} := \mathbf{z}^*
    \end{aligned}
\end{equation}
where $\mathbf{z}^* \in \mathbb{R}^{d_z}$ is the fixed point, $\mathbf{x}\in \mathbb{R}^{d_x}$ is a data point and $f_\theta : \mathbb{R}^{d_z} \times \mathbb{R}^{d_x} \rightarrow \mathbb{R}^{d_z}$ is a neural network with parameters $\theta$.

The output $\mathbf{y}$ is defined implicitly by $f_\theta$. An algorithm (fixed point solver) is therefore required to compute the fixed point $\mathbf{z}^*$. The simplest fixed point solver is obtained by iterating $f_\theta$,
\begin{equation}
    \label{eq:fixed-point-iteration}
    \mathbf{z}_{n+1} = f_\theta(\mathbf{z}_n, \mathbf{x}), \qquad \mathbf{z}_0 = \vec{0},
\end{equation}
for $0\leq n \leq N$, where $N$ defines the maximum number of steps. We see that equation \eqref{eq:fixed-point-iteration} looks similar to an explicit feed-forward network, where the input is injected at each step $n$. Alternatively, we do not need to fix the maximum number of steps but can rather iterate equation \eqref{eq:fixed-point-iteration} until convergence, $\lVert \mathbf{z}_{n+1} - \mathbf{z}_n \rVert < \epsilon$ for some tolerance $\epsilon$.

In practice, stronger fixed point solvers are used to accelerate convergence. Broyden's method \cite{broyden1965class}, a quasi-Newton root finding algorithm, was the original solver of choice \cite{bai2019deep}. Recently, Anderson's method \cite{anderson1965iterative} has become popular due to the reduced computational cost and improved performance \cite{gilton2021deep}.

\subsection{Backpropagation}
One key property of implicit models is that an associated adjoint system can be derived and solved numerically to find the gradient. This removes the need to store the forward computation graph and leads to constant memory efficiency with respect to the graph depth.

For DEQs, the gradient can be derived via the Implicit Function Theorem (IFT) \cite{bai2019deep, winston2020monotone}. Consider a scalar-valued loss function $L(z^*) : \mathbb{R}^{d_z}\rightarrow \mathbb{R}$ and associated gradient (adjoint) $\mathbf{a}_{z} = \partial L/\partial \mathbf{z}$. Then,
\begin{equation}
    \frac{\partial L}{\partial \theta} = \left(\frac{\partial f_\theta(\mathbf{z}^*, \mathbf{x})}{\partial \theta}\right)^T \mathbf{g},
\end{equation}
where $\mathbf{g}$ is obtained by solving the gradient fixed-point system,
\begin{equation}
    \label{eq:adjoint-system}
    \mathbf{g} = \left(\frac{\partial f_\theta(\mathbf{z}^*, \mathbf{x})}{\partial \mathbf{z}^*}\right)^T \mathbf{g} + \mathbf{a}_z.
\end{equation}
That is, equation \eqref{eq:adjoint-system} can be solved analogously to equation \eqref{eq:deq} by solving for the fixed point $\mathbf{g}$.

While the memory efficiency of this approach is attractive, the resulting gradient is a numerical approximation. To obtain a strong approximation the fixed point tolerance, $\epsilon$, must be small and a large number of solver steps, $N$, are required.

Given the difficulty of training implicit models via the implicit function theorem, many alternative gradient approximations can be made: Neumann series \cite{lorraine2020optimizing, geng2021training}, phantom gradient \cite{geng2021training} and Jacobian-free backpropagation \cite{fung2022jfb}. Note that these methods introduce a further approximation to the exact gradient of the model but are cheaper to compute.

\subsection{Algebraic Reversibility}
Alternatively, it is possible to compute exact gradients while retaining the memory efficiency of backpropagation via the IFT. This is achieved by constructing an algebraically reversible fixed point solver, whereby the solver is able to step between $\mathbf{z}_n$ and $\mathbf{z}_{n+1}$ in closed form.

Consider the fixed point iteration in equation \eqref{eq:fixed-point-iteration}. At runtime, the iteration unrolls $N$ steps to form the forward computation graph. This discrete, finite-precision graph \emph{is} the model. The gradient of this graph corresponds to the model being trained and is therefore referred to as exact.

We could simply store this graph in memory and backpropagate using standard automatic differentiation \cite{griewank2008evaluating}. However, the memory cost of this graph scales with $N$. Instead, if our fixed point solver is algebraically reversible then the forward graph can be reconstructed exactly (up to floating point precision) during backpropagation. By reconstructing the forward graph exactly, gradient calculation becomes exact.

This is the idea we present here. In Section \ref{sec:revdeq}, we construct an algebraically reversible fixed point solver and derive the exact backpropagation algorithm. This algorithm is linear in time $O(N)$ and constant in memory $O(1)$. We then show in Section \ref{sec:experiments} that the resulting model, termed a Reversible Deep Equilibrium Model (RevDEQ), significantly outperforms previous DEQ approaches and comparable explicit architectures. 

Reversibility has been found to be a useful idea in many areas of machine learning: reversible networks apply this to residual networks \cite{gomez2017reversible, chang2018reversible}, ReFormer uses this construction in transformers \cite{kitaev2020reformer}, momentum networks present an alternative reversible residual network \cite{sander2021momentum}, coupling layers are used in generative modelling for exact likelihood calculation \cite{dinh2014nice} and image editing \cite{wallace2023edict}, reversible neural differential equation solvers improve backpropagation efficiency \cite{mccallum2024efficient, kidger2021efficient, zhuang2021mali}.

\section{Reversible Deep Equilibrium Models}
\label{sec:revdeq}

\begin{figure}
    \centering
    \includegraphics[width=0.8\linewidth]{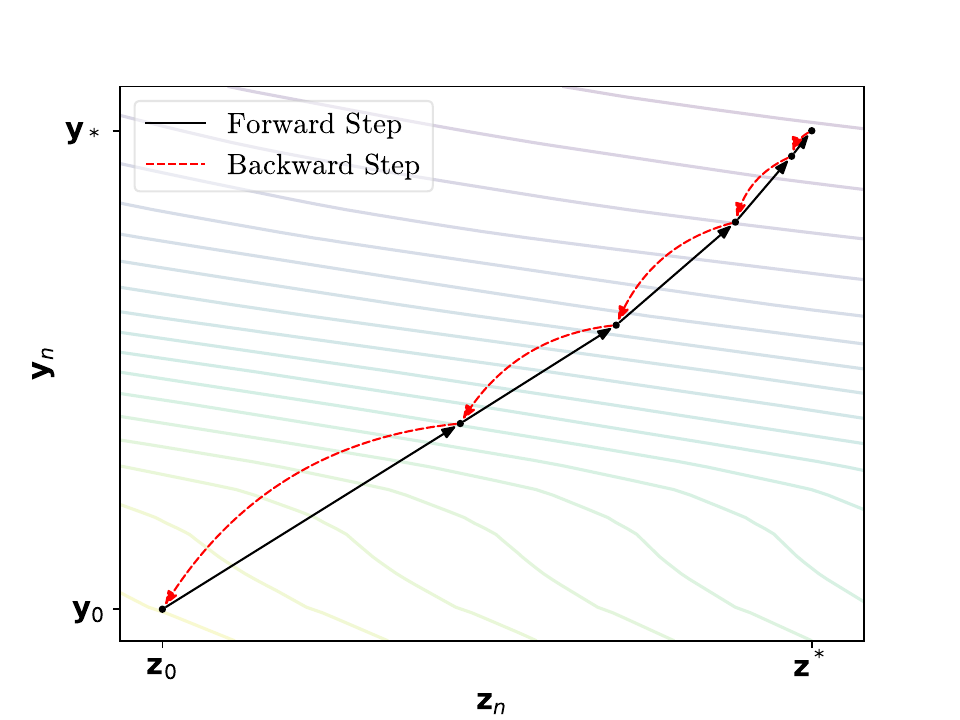}
    \caption{Example of the forward and backward passes in RevDEQ. Starting at $\{\mathbf{y}_0, \mathbf{z}_0\}$ we iterate the forward step until we reach (an approximation to) the fixed point $\{\mathbf{y}^*, \mathbf{z}^*\}$. On the backward pass we reverse each forward step until we return to the initial condition.}
    \label{fig:overview}
\end{figure}

\subsection{Forward Pass}
Consider the fixed point system in equation \eqref{eq:deq}. The forward pass of a RevDEQ is given by the following update rule,
\begin{equation}
    \label{eq:reversible-fixed-point}
    \begin{aligned}
    \mathbf{y}_{n+1} &= (1-\beta)\mathbf{y}_n + \beta f_\theta(\mathbf{z}_n, \mathbf{x}), \\ 
    \mathbf{z}_{n+1} &= (1-\beta)\mathbf{z}_n + \beta f_\theta(\mathbf{y}_{n+1}, \mathbf{x}),
    \end{aligned}
\end{equation}
where $\mathbf{y}_n, \mathbf{z}_n$ defines the coupled fixed point state with initial condition given by $\mathbf{z}_0=\mathbf{y}_0=\vec{0}$ and $0< \beta < 2$ is a relaxation parameter. Each update of equation \eqref{eq:reversible-fixed-point} performs one relaxed fixed point iteration. But, the second update evaluates the function $f_\theta$ at the updated state $\mathbf{y}_{n+1}$. This couples $\mathbf{z}_n, \mathbf{y}_n$ together and creates an algebraically reversible step.

\subsection{Backward Pass}
Since the update in equation \eqref{eq:reversible-fixed-point} is reversible, we can simply invert the multiplication and addition operations to exactly reverse one forward step. This results in the backward step, given by
\begin{equation}
    \label{eq:reversible-backward-step}
    \begin{aligned}
    \mathbf{z}_n &= \frac{\mathbf{z}_{n+1}-\beta f_\theta(\mathbf{y}_{n+1}, \mathbf{x})}{1-\beta}, \\ 
    \mathbf{y}_n &= \frac{\mathbf{y}_{n+1}-\beta f_\theta(\mathbf{z}_n, \mathbf{x})}{1-\beta}.
    \end{aligned}
\end{equation}
It is this reversibility property that allows for exact gradient calculation; the forward pass can be exactly reconstructed during backpropagation (see Figure \ref{fig:overview}). We derive the reversible backpropagation algorithm below in Section \ref{subsec:backpropagation}.

\subsection{Convergence}
Let $\mathbf{q}^*$ denote the unique fixed point of $f_\theta(\cdot, \mathbf{x})$. Then the reversible fixed point solver in equation \eqref{eq:reversible-fixed-point} converges to the fixed point $\mathbf{q}^*$ with $\mathbf{y}^*=\mathbf{z}^*=\mathbf{q}^*$. That is, both states $\mathbf{y}_n$ and $\mathbf{z}_n$ converge to the unique fixed point of $f_\theta$, as $n\rightarrow \infty$. This is illustrated by Theorem \ref{thm:reversible-fixed-point}.

\begin{restatable}{theorem}{revthm}
    \label{thm:reversible-fixed-point}
    Let $\{\mathbf{y}_n, \mathbf{z}_n\}_{n \geq 0}$ be the sequence obtained by the reversible iteration in equation \eqref{eq:reversible-fixed-point} and consider a contractive function $f: \mathbb{R}^d \rightarrow \mathbb{R}^d$ with Lipschitz constant $0 < k < 1$. Then, for all $0<\beta<2/(k+1)$ the pair $\mathbf{y}_n, \mathbf{z}_n$ both converge to the unique fixed point $\mathbf{q}^*=f(\mathbf{q}^*)$.
\end{restatable}

\begin{proof}
    The proof of Theorem \ref{thm:reversible-fixed-point} is given in Appendix \ref{app:convergence}.
\end{proof}

The reversible scheme converges linearly in the number of steps $N$ with constant $L = |1-\beta| + \beta k$, such that $\lVert \mathbf{y}_N - \mathbf{q}^*\rVert \leq L^N \lVert \mathbf{y}_0 - \mathbf{q}^*\rVert$. The reversible scheme therefore shares the same convergence speed as the single-step relaxed fixed point iteration. This is understood to result from the reversible convergence speed being dominated by the first state updated, $\mathbf{y}_n$, which follows the relaxed update rule.

\begin{minipage}[t]{0.48\textwidth}
    \subsection{Backpropagation}
    \label{subsec:backpropagation}
    \vspace{1ex}
        
    The reversible backpropagation algorithm proceeds by first reconstructing the state $\{\mathbf{y}_n, \mathbf{z}_n\}$ from $\{\mathbf{y}_{n+1}, \mathbf{z}_{n+1}\}$ using the backward step in equation \eqref{eq:reversible-backward-step}. Then, the gradients are backpropagated by following the reverse-mode automatic differentiation rules \cite{griewank2008evaluating} for the forward step in equation \eqref{eq:reversible-fixed-point}. See Algorithm \ref{alg:backprop-fixed-point}.

    \vspace{1ex}
        
    We define a scalar-valued loss function $L(\mathbf{z}_N) : \mathbb{R}^d \rightarrow \mathbb{R}$ on the terminal value of the solve $\mathbf{z}_N \approx \mathbf{z^*}$. The gradients (adjoints) are defined as $\bar{\mathbf{v}}=\partial L(\mathbf{z}_N)/\partial \mathbf{v}$. The presence of an adjoint to the left of a Jacobian, $\bar{\mathbf{v}} \frac{\partial f_\theta}{\partial \mathbf{z}}$, indicates a vector-Jacobian product. This can be efficiently computed using automatic differentiation \cite{griewank2008evaluating, griewank1992achieving}.

    \vspace{1ex}
        
    Algorithm \ref{alg:backprop-fixed-point} implements the same backpropagation rules as would be taken by reverse-mode automatic differentiation with the forward graph stored in memory. However, instead of storing the forward graph we can reconstruct it exactly from the terminal state $\{\mathbf{y}_N, \mathbf{z}_N\}$. The algorithm therefore has constant memory complexity $O(1)$ and linear time complexity $O(N)$ with respect to the depth of the graph $N$.
\end{minipage}
\hfill
\begin{minipage}[t]{0.48\textwidth}
\begin{algorithm}[H]
    \captionof{algorithm}{Reversible Backpropagation}\label{alg:backprop-fixed-point}
    \begin{algorithmic}
        \State {\bfseries Input:} $\mathbf{y}_{n+1}, \mathbf{z}_{n+1}, \bar{\mathbf{y}}_{n+1}, \bar{\mathbf{z}}_{n+1}, \bar{\theta}$ 
        \vspace{1ex}
        \State \# Backward step
        \vspace{1ex}
        \State $\mathbf{z}_n = \displaystyle\frac{\mathbf{z}_{n+1} - \beta f_\theta(\mathbf{y}_{n+1}, \mathbf{x})}{1-\beta}$
        \vspace{1ex}
        \State $\mathbf{y}_n = \displaystyle\frac{\mathbf{y}_{n+1} - \beta f_\theta(\mathbf{z}_n, \mathbf{x})}{1-\beta}$
        \vspace{2ex}
        \State \# Gradients
        \State $\bar{\mathbf{y}}_{n+1} \leftarrow \bar{\mathbf{y}}_{n+1} + \beta\bar{\mathbf{z}}_{n+1}\displaystyle\frac{\partial f_\theta(\mathbf{y}_{n+1}, \mathbf{x})}{\partial \mathbf{y}_{n+1}}$
        \vspace{1ex}
        \State $\bar{\mathbf{y}}_n = (1-\beta)\bar{\mathbf{y}}_{n+1}$
        \vspace{1ex}
        \State $\bar{\mathbf{z}}_n = (1-\beta)\bar{\mathbf{z}}_{n+1} + \beta \bar{\mathbf{y}}_{n+1}\displaystyle\frac{\partial f_\theta(\mathbf{z}_n, \mathbf{x})}{\partial \mathbf{z}_n}$
        \vspace{1ex}
        \State $\begin{aligned}\bar{\theta} = \bar{\theta} &+ \beta \bar{\mathbf{z}}_{n+1}\displaystyle\frac{\partial f_\theta(\mathbf{y}_{n+1}, \mathbf{x})}{\partial \theta} \\ &+ \beta \bar{\mathbf{y}}_{n+1}\displaystyle\frac{\partial f_\theta(\mathbf{z}_n, \mathbf{x})}{\partial \theta}\end{aligned}$
        \vspace{1ex}
        \State {\bfseries Return:} $\mathbf{y}_n, \mathbf{z}_n, \bar{\mathbf{y}}_n, \bar{\mathbf{z}}_n, \bar{\theta}$
    \end{algorithmic}
\end{algorithm}
\end{minipage}

\section{Experiments}
\label{sec:experiments}
We demonstrate the performance of RevDEQ on language modelling and image classification tasks. 

For language modelling, we use the Wikitext-103 dataset containing over 100 million tokens from the set of verified good and featured articles on Wikipedia \cite{merity2016pointer}. We compare performance to previous DEQ models \cite{bai2019deep, bai2021stabilizing} and an explicit transformer architecture via Transformer-XL \cite{dai2019transformerxl}.

For image classification, we use the CIFAR-10 dataset \cite{krizhevsky2009learning}. We introduce a novel multi-scale implicit architecture, inspired by the explicit ResNet model \cite{he2016deep}, where the traditional deep unrolled convolutional layers at each image scale are replaced by a single RevDEQ block. This model is significantly simpler than previous multi-scale DEQ approaches \cite{bai2020multiscale, winston2020monotone} and achieves higher classification accuracy. We compare performance to previous single-scale and multi-scale implicit models, and explicit ResNet architectures.

\subsection{Language Modelling on Wikitext-103}
\label{subsec:wikitext}
The original DEQ model was applied to model sequential data, such as language \cite{bai2019deep}. This was motivated by the observation that the hidden layers of explicit deep architectures naturally learn contractive functions and push the hidden state towards a fixed point. 

We evaluate the performance of the RevDEQ model on the Wikitext-103 dataset \cite{merity2016pointer}. The model is parameterised with a decoder-only transformer architecture \cite{vaswani2017attention}, but with the many explicit layers replaced by a single layer, $f_\theta$, of which we find the fixed point. 

\paragraph{Architecture} Given a sequence of (positional plus token) embeddings $\mathbf{x}_{1:T}$, with sequence length $T$ and each element $\mathbf{x}_i\in \mathbb{R}^d$, the equilibrium module $f_\theta$ is given by

\begin{equation}
    \label{eq:transformer}
    \begin{aligned}
        \hat{\mathbf{z}}_{1:T} &= \mathbf{z}_{1:T} + \mathbf{x}_{1:T}, \\
        f_\theta(\mathbf{z}_{1:T}, \mathbf{x}_{1:T}) &= \text{LN} \circ \phi \circ \text{LN} \circ \text{SelfAttention}(W_{QKV}\hat{\mathbf{z}}_{1:T}).
    \end{aligned}
\end{equation}

First the input sequence is added to the hidden state $\mathbf{z}_{1:T}$, then we apply a self-attention operation with $W_{QKV}\in \mathbb{R}^{3d \times d}$ \cite{vaswani2017attention} followed by layer normalisation (LN) \cite{ba2016layer}, 2-layer MLP $\phi$, and a final LN. The self-attention operation is multi-headed, with $8$ heads used for the experiments below.

The fixed point of $f_\theta$ is calculated using the reversible fixed point solver in equation \eqref{eq:reversible-fixed-point}. Gradient calculation is performed via the reversible backpropagation algorithm, given by Algorithm \ref{alg:backprop-fixed-point}. We set the fixed point tolerance to be $10^{-3}$ and choose a maximum number of solver steps. See Appendix \ref{app:experiments} for full details.

We compare RevDEQ performance to comparable implicit and explicit models in Table \ref{tab:perplexity-wikitext}. This includes the Transformer-XL explicit architecture \cite{dai2019transformerxl}, the original DEQ-Transformer \cite{bai2019deep} and the DEQ-Transformer with Jacobian regularisation (JR) \cite{bai2021stabilizing}. 

\begin{table}[b]
    \centering
    \renewcommand{\arraystretch}{1.3}
    \begin{tabular}{c c c c}
         \toprule
         Model & Parameters & Function Evaluations & Perplexity \\
         \midrule
         Transformer-XL (18 layers) \cite{dai2019transformerxl} & 110M & - & 24.1 \\
         DEQ-Transformer \cite{bai2019deep} & 98M & 30 & 24.0 \\
         DEQ-Transformer + JR \cite{bai2021stabilizing} & 98M & 14 & 24.5 \\
         DEQ-Transformer + JR \cite{bai2021stabilizing} & 98M & 12 & 24.9 \\
         \midrule
         \textbf{RevDEQ} & 94M & \textbf{8} & 23.4 \\
         \textbf{RevDEQ} & 94M & 12 & \textbf{23.2} \\
         \bottomrule
    \end{tabular}
    \vspace{1ex}
    \caption{Test perplexity on the Wikitext-103 dataset. RevDEQ models outperform previous DEQ approaches while requiring fewer function evaluations. The RevDEQ model also outperforms explicit transformer architectures of comparable size.}
    \label{tab:perplexity-wikitext}
\end{table}

\vspace{2ex}

\paragraph{Comparison to implicit models} From Table \ref{tab:perplexity-wikitext} we see that the RevDEQ model significantly outperforms the original DEQ model while using $3.75\times$ fewer function evaluations. The larger number of function evaluations of the DEQ model results from the difficulty in computing the gradient via the implicit function theorem. The RevDEQ model avoids this difficulty and computes an exact gradient regardless of the number of function evaluations. 

We hypothesise that the improved performance is largely a function of exact gradient calculation. To test this, we vary the number of function evaluations in the RevDEQ model and calculate test perplexity. This is shown in Table \ref{tab:nfe-wikitext}.

It can be seen that the RevDEQ model is competitive with previous DEQ models when using only 6 function evaluations. Additionally, the perplexity quickly decreases as the number of function evaluations increase and plateaus after approximately 12 evaluations. This scaling is much faster than previous implicit models, where typically 10s or 100s of evaluations are required to reach optimality. 

\paragraph{Comparison to explicit models} Furthermore, we find that the RevDEQ model outperforms explicit transformer architectures for comparable model sizes. Previous DEQ models were only able to approximately match performance. Here, we show that implicit architectures can indeed surpass explicit model performance in language modelling tasks.

This result can be understood by considering the following argument. The explicit transformer architecture distributes the learnable parameters over $N$ layers. But it is known that an $N$ layer deep network can be represented by a single weight-tied layer that is iterated for (at most) $N$ steps, where the single hidden state is formed by the concatenation of the hidden states over the $N$ layers (see Theorem 3 of \cite{bai2019deep}). 

This reasoning is instantiated by DEQs, where the parameter count of the single layer $f_\theta$ is increased to match the parameter count of the $N$ deep layers. However, we are now free to extend the number of steps taken by the model over $N$. This effectively increases the depth of the model and therefore can result in improved performance, as is observed here.

\begin{table}
    \centering
    \renewcommand{\arraystretch}{1.3}
    \begin{tabular}{c c c}
         \toprule
         Model & Function Evaluations & Perplexity \\
         \midrule
         DEQ-Transformer \cite{bai2019deep} & 30 & 24.0 \\
         DEQ-Transformer (JR) \cite{bai2021stabilizing} & 12 & 24.9 \\
         \midrule
         \multirow{5}{*}{\textbf{RevDEQ}} & \textbf{2} & 27.3 \\
         & 4 & 24.6 \\
         & 6 & 23.7 \\
         & 8 & 23.4 \\
         & 10 & 23.3 \\
         & 12 & \textbf{23.2} \\
         \bottomrule
    \end{tabular}
    \vspace{1ex}
    \caption{Test perplexity on the Wikitext-103 dataset in terms of the number of function evaluations. RevDEQ models are more compute efficient in comparison to previous DEQs.}
    \label{tab:nfe-wikitext}
\end{table}

\subsection{Image Classification on CIFAR-10}
Next we consider the application of the RevDEQ model to image classification on CIFAR-10. We compare to previous single-scale (one image resolution) and multi-scale (many image resolutions) implicit and explicit models.

Previous multi-scale implicit architectures have focused on concatenating (width-wise) all scales into one feature vector where each scale is driven to equilibrium in parallel \cite{bai2020multiscale, winston2020monotone}. Here, we instead choose to model each image scale with an implicit architecture and parameterise explicit downsampling layers between scales. This is a ResNet model \cite{he2016deep}, but with the many deep unrolled convolutional blocks at each image scale replaced by a single convolutional block of which we find the fixed point.

This parameterisation is much simpler than previous multi-scale architectures \cite{bai2020multiscale, winston2020monotone}; additionally, we find the model to be more parameter efficient and performant, even in comparison to explicit ResNet models \cite{he2016deep}.

\paragraph{Architecture} The single-scale residual block, $f_\theta(\mathbf{z}, \mathbf{x})$ is defined by
\begin{equation}
    \label{eq:single-scale-deq}
    \begin{aligned}
    \hat{\mathbf{z}} &= \text{ReLU} \circ \text{norm}(W_1 * \mathbf{z}), \\
    \tilde{\mathbf{z}} &= \text{ReLU} \circ \text{norm}(\mathbf{x} + W_2 * \hat{\mathbf{z}}), \\
    f_\theta(\mathbf{z}, \mathbf{x}) &= \text{norm}(\mathbf{z} + \tilde{\mathbf{z}}),
    \end{aligned}
\end{equation}
where $W_1, W_2$ are weight matrices, $*$ denotes a convolution operation, $\text{norm}$ stands for normalisation (group normalisation is used here) and $\text{ReLU}$ indicates the rectified linear unit activation function. The application of the reversible fixed point algorithm to equation \eqref{eq:single-scale-deq} defines the output - we refer to this as a RevDEQ block. 

\noindent
\begin{minipage}[t]{0.55\textwidth}
    \paragraph{Single-scale} For the single-scale model we simply encode the input image, compute the RevDEQ block and apply the final linear classification layers. See Appendix \ref{app:experiments} for full details.

    \vspace{2ex}
    
    \paragraph{Multi-scale} For the multi-scale model, we first encode the input image, then apply a RevDEQ block and downsample the output. At the next feature scale we again apply a RevDEQ block and downsample. See Figure \ref{fig:implicit-resnet-diagram} for an illustration. 

    \vspace{1ex}
    
    This is iterated four times, defining four feature scales as is common in ResNet models \cite{he2016deep}. At each scale, $i$, the number of channels $C_i$ is selected along with the fixed point tolerance $\epsilon_i$ and maximum steps $N_i$. The number of steps $N_i$ defines the effective depth of the model at each scale.
    
    \vspace{1ex}

    The downsampling layers follow the design of \cite{he2016identity}, with the image dimension $(H_i, W_i)$ halved by a strided convolution operation. The residual connection follows a 1x1 convolution to align the channel dimensions $C_i \rightarrow C_{i+1}$. Finally, we apply the linear classification layers. See Appendix \ref{app:experiments} for full details.

\end{minipage}
\hfill
\begin{minipage}[t]{0.4\textwidth}
    \centering
    \begin{tikzpicture}[>=stealth,thick, baseline=(current bounding box.north)]
        \node (x) {$\mathbf{x}_i$};
        \node[draw, minimum width=2.0cm, minimum height=0.8cm, below=8mm of x, align=center] (revdeq) {RevDEQ \\ $(C_i, H_i, W_i)$};
        \node[circle, draw, inner sep=1pt, below=5mm of revdeq] (add1) {+};
        \node[draw, minimum width=2cm, minimum height=0.8cm, below=8mm of add1, align=center] (down) {Downsample \\ $C_i\rightarrow C_{i+1}$};
        \node[circle, draw, inner sep=1pt, below=5mm of down] (add2) {+};
        \node[below=0.5cm of add2] (output) {$\mathbf{x}_{i+1}$};
        
        \draw[->] (x) -- (revdeq);
        \coordinate (midarrow1) at ($ (x)!0.3!(revdeq) $);
        \coordinate (midarrow2) at ($ (add1)!0.3!(down) $);
        \draw[->] (revdeq) -- (add1);
        \draw[->] (add1) -- (down);
        \draw[->] (down) -- (add2);
        \draw[->] (midarrow1.east) to[out=0,in=0, looseness=2.5] (add1.east);
        \draw[->, dashed] (midarrow2.east) to[out=0,in=0, looseness=2.5] (add2.east);
        \draw[->] (add2) -- (output);
    \end{tikzpicture}
    
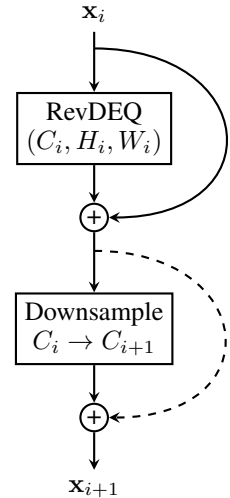
\captionof{figure}{A single scale of the multi-scale implicit ResNet architecture.}
    \label{fig:implicit-resnet-diagram}
\end{minipage}

\paragraph{Results} The results are shown in Table \ref{tab:cifar10}. For the single-scale architectures, the RevDEQ model outperforms previous implicit architectures while using fewer parameters and function evaluations.

The multi-scale RevDEQ model (small, 170K parameters), achieves a higher classification accuracy than all previous multi-scale architectures. Compared to the multi-scale DEQ (MDEQ) with the same number of parameters, the RevDEQ model uses $3\times$ fewer functions evaluations.

The medium size multi-scale models is the most interesting category. The RevDEQ model (5M parameters) obtains a higher classification accuracy than the explicit ResNet-18 model (10M parameters) and matches classification accuracy of the explicit ResNet-101 model (40M parameters). This result can be understood by the argument presented in Section \ref{subsec:wikitext} (comparison to explicit models).

Furthermore, the RevDEQ model matches performance of the MDEQ architecture with half the parameter count and $3\times$ fewer functions evaluations. For the same model size, the RevDEQ model achieves a higher classification accuracy and still maintains a significant reduction in function evaluations ($1.9\times$). 

Overall, these results indicate that the RevDEQ model can be used as a modular block that replaces explicit unrolled deep layers in ResNet type models; higher classification accuracy is obtained with the implicit-depth block while improving on parameter efficiency. This drop-in replacement is significantly simpler than previous tailor-made implicit vision architectures.

\begin{table}[t]
    \centering
    \renewcommand{\arraystretch}{1.3}
    \begin{tabular}{c c c c}
         \toprule
         Model & Parameters & Function Evaluations & Accuracy (\%) \\
         \midrule
         \multicolumn{4}{c}{\emph{Single-scale}} \\
         DEQ \cite{bai2019deep} & 170K & 15 & $82.2\pm 0.3$ \\
         monDEQ \cite{winston2020monotone} & 854K & 20 & $82.0 \pm 0.3$ \\
         \textbf{RevDEQ} & \textbf{170K} & \textbf{8} & $\mathbf{87.5 \pm 0.6}$ \\
         \midrule
         \multicolumn{4}{c}{\emph{Multi-scale (small)}} \\
         MDEQ \cite{bai2020multiscale} & 170K & 15 & $87.1 \pm 0.4$ \\
         monDEQ \cite{winston2020monotone} & 1M & 20 & $89.0 \pm 0.3$ \\
         pcDEQ \cite{gabor2024positive} & 661K & 14 & $89.2 \pm 0.0$ \\
         \textbf{RevDEQ} & \textbf{170K} & $\mathbf{5}$ & $\mathbf{89.6 \pm 0.3}$ \\
         \midrule
         \multicolumn{4}{c}{\emph{Multi-scale (medium)}} \\
         ResNet-18 \cite{he2016deep} & 10M & - & $92.9 \pm 0.2$ \\
         ResNet-101 \cite{he2016deep} & 40M & - & $93.8 \pm 0.3$ \\
         MDEQ \cite{bai2020multiscale} & 10M & 15 & $93.8 \pm 0.2$ \\
         MDEQ (JR) \cite{bai2021stabilizing} & 10M & 6 & $93.1 \pm 0.3$ \\
         \textbf{RevDEQ} & \textbf{5M} & \textbf{5} & $\mathbf{93.8 \pm 0.1}$ \\
         \textbf{RevDEQ} & \textbf{10M} & \textbf{8} & $\mathbf{94.4 \pm 0.1}$ \\
         \bottomrule 
    \end{tabular}
    \vspace{1ex}
    \caption{Test accuracy on the CIFAR-10 dataset. RevDEQ outperforms comparable implicit and explicit models, across a range of architectures and model sizes, while using fewer function evaluations. For multi-scale models, the number of function evaluations is taken as the mean across scales.}
    \label{tab:cifar10}
\end{table}
 
\section{Discussion}
\subsection{Performance improvements}
\paragraph{Mixed precision} The reconstruction accuracy obtained when reversing an addition operation in floating point arithmetic is limited by the numerical precision of the operands. This is due to information loss in the mantissa representation after a bit shift to align the two operand exponents. We therefore found it to be beneficial to use a mixed-precision strategy, where the addition/subtraction operations of the reversible fixed point solver are performed in 64-bit precision and all other operations are in 32-bit precision (or lower).

\paragraph{Choice of $\beta$} In the reversible fixed point scheme, the backward step features a division by $1-\beta$. This amplifies the (finite-precision) error in the subtraction step of the algorithm. Reducing $\beta$ can therefore improve gradient accuracy at the cost of convergence speed (increased damping). We found that values $\beta \in [0.5, 0.9]$ worked well for all tasks considered.

\subsection{Limitations}
\paragraph{Speed of computation} The RevDEQ model significantly improves on runtime in comparison to previous implicit models due to the reduced number of function evaluations required. However, the runtime can still be greater than comparable explicit models with fewer effective function evaluations. We believe this to be largely an implementation issue as implicit models have the potential to be more GPU efficient than their explicit counterparts (see Section \ref{subsec:future-work}).

\paragraph{Stateful normalisation} The normalisations used inside the equilibrium function $f_\theta$ are stateless layer-wise. This is due to the number of layers in the model being defined implicitly, making it challenging to collect population statistics for each `layer'. In the implicit ResNet model, to obtain stateful normalisation over the hidden state we use batch normalisation \cite{ioffe2015batch} in the downsampling blocks prior to a RevDEQ block.

\subsection{Future work}
\label{subsec:future-work}
\paragraph{Applications} There are many applications of implicit architectures in a diverse set of tasks, such as graph neural networks \cite{gu2020implicit}, flow \cite{lu2021implicit, bai2022deep} and diffusion models \cite{pokle2022deep}, implicit neural representations \cite{huang2021textrm} and inverse problems \cite{gilton2021deep}. The RevDEQ model is yet to be applied to these problems!

\paragraph{GPU efficiency} The arithmetic intensity of modern GPUs is typically $\sim 10^2$, meaning that if less than $10^2$ FLOPs are performed per byte of data then algorithm speed is limited by memory operations \cite{scaling-book}. In standard backpropagation, many read/write operations are required to store and load the forward pass of the network \cite{griewank2008evaluating, dao2022flashattention}. For implicit models, this is not required - completely removing these memory operations. By careful algorithm design, this could be exploited to improve the GPU efficiency of implicit models.

\section{Conclusion}
We have introduced a new class of reversible implicit model, RevDEQs. The model uses a reversible fixed point algorithm to compute the forward pass and performs exact gradient backpropagation in constant memory with respect to model depth. This construction demonstrates improved performance while significantly reducing the large number of function evaluations previously required by implicit architectures.

\section*{Acknowledgements}
SM and KA were supported by the EPSRC grant EP/S022945/1. JF was supported by the Department of Mathematical Sciences at the University of Bath.

\newpage
\bibliographystyle{unsrt}
\bibliography{bibliography}

\newpage
\section*{Appendix}
\appendix

\section{Convergence}
\label{app:convergence}

\subsection{Notation and definitions}

Let $\norm{\cdot}$ denote an arbitrary norm on $\bRd$. We then define $\norm{\cdot}_\infty$ as the $\norm{\cdot}$-induced max norm on $\bRd \times \bRd$, that is, 
$$
\norm{(\bx, \by)}_{\infty} \coloneqq \max \left \{ \norm{\bx}, \norm{\by}\right \}, \quad (\bx,\by) \in \bRd \times \bRd.
$$

Recall the reversible fixed point iteration \eqref{eq:reversible-fixed-point}:
\begin{align*}
\mathbf{y}_{n+1} &= (1-\beta)\mathbf{y}_n + \beta f_\theta(\mathbf{z}_n, \mathbf{x}), \\ 
\mathbf{z}_{n+1} &= (1-\beta)\mathbf{z}_n + \beta f_\theta(\mathbf{y}_{n+1}, \mathbf{x}).
\end{align*}

Without loss of generality, we can drop the $\theta$- and $\mathbf{x}$-dependence of $f_\theta(\cdot, \mathbf{x})$ and consider $f : \bRd \rightarrow \bRd$. We then define a map $\varphi$ on $\bRd \times \bRd$ via
$$
\varphi(\by, \bz) \coloneqq \left((1-\beta)\by + \beta f(\bz), (1-\beta)\bz + \beta f((1-\beta)\by + \beta f(\bz)) \right), \quad (\by, \bz) \in \bRd \times \bRd.
$$
Observe that $\varphi$ is the map associated to the reversible fixed point iteration, that is, 
$$
(\by_{n+1}, \bz_{n+1}) = \varphi(\by_n, \bz_n).
$$

In the following, we will make use of the Banach fixed point theorem. We recall it here for completeness.

\begin{theorem*}[Banach fixed point theorem]
Let $(X, d)$ be a non-empty complete metric space and assume $T: X \rightarrow X$ is a contraction. That is, there exists $k \in [0, 1)$ such that 
$$
d ( T\bx, T\by) \leq k d(\bx, \by),
$$
for all $\bx, \by \in X$. Then $T$ has a unique fixed point $\bx^* \in X$ (i.e. $\bx^* = T \bx^*$). Moreover, the iteration $\bx_{n+1} = T \bx_n$ converges to $\bx^*$ for any $\bx_0 \in X$.
\end{theorem*}

\subsection{Convergence analysis}

In this section, we prove Theorem \ref{thm:reversible-fixed-point}, and derive the rate and speed of convergence of the reversible fixed point iteration. We first recall the theorem:

\revthm*

\begin{proof}
    Directly follows from Lemma \ref{lem:fixed-point} and Lemma \ref{lem:reversible-contractivity}, which are stated and proved below.
\end{proof}

\begin{lemma}\label{lem:fixed-point}
    Suppose that $f$ is a contraction with contraction constant $k \in [0, 1)$ and fixed point $\bq^*$. Additionally, assume that $\varphi$ admits a unique fixed point $(\by^*, \bz^*)$. Then $\by^*  =\bz^* = \bq^*$.
\end{lemma}

\begin{proof}
    The fixed point of $\varphi$ satisfies $(\by^*, \bz^*) = \varphi(\by^*, \bz^*)$. This is equivalent to the equations
    \begin{align*}
        \by^* &= (1-\beta)\by^* + \beta f(\bz^*), \\
        \bz^* &= (1-\beta)\bz^* + \beta f(\by^*).
    \end{align*}
    Upon simplification, this is the same as 
    $$
    \by^* = f(\bz^*), \quad \text{and} \quad \bz^* = f(\by^*).
    $$
    Applying $f$ to both equations we find that $\by^* = (f \circ f)(\by^*)$ and $\bz^* = (f \circ f)(\bz^*)$. That is, both $\by^*$ and $\bz^*$ are fixed points of $f \circ f$. Since $f$ is a contraction, $f \circ f$ is also a contraction (with contraction constant $k^2$). Therefore, by Banach's fixed point theorem, $f \circ f$ admits a unique fixed point giving $\by^* = \bz^*$. Moreover, if $\bq^*$ is the unique fixed point of $f$ then it is necessarily the unique fixed point of $f \circ f$ too, giving $\by^* = \bz^* = \bq^*$ as claimed.
\end{proof}

We now show that $\varphi$ is contractive in an appropriately chosen norm. Existence of the fixed point then follows by the Banach fixed point theorem. The sequential structure of the reversible fixed point iteration coupled with the diagonal nature of the fixed point motivates proving contractivity in $\norm{\cdot}_\infty$. In particular, we show contractivity of each component of the iteration in the arbitrary norm $\norm{\cdot}.$

\begin{lemma}\label{lem:reversible-contractivity}
    Assume that $f$ is a contraction with contraction constant $k \in [0, 1)$. Then, for all $0 < \beta < 2/(k+1)$, $\varphi$ is a contraction under $\norm{\cdot}_\infty$ with contraction constant $L = \lvert 1-\beta\rvert + \beta k.$
\end{lemma}

\begin{proof}
    We want to show the existence of $L \in [0, 1)$ such that
    \begin{equation}\label{eq: claim}
    \norm{\varphi(\by, \bz) - \varphi(\by', \bz')}_\infty \leq L \norm{(\by, \bz) - (\by', \bz')}_\infty \quad \forall~ (\by, \bz), (\by', \bz') \in \bRd \times \bRd.
    \end{equation}
    Write
    $$
    \norm{\varphi(\by, \bz) - \varphi(\by', \bz')}_\infty = \max \{\underbrace{\norm{\varphi_1(\by, \bz) - \varphi_1(\by', \bz')}}_{T_1}, \underbrace{\norm{\varphi_2(\by, \bz) - \varphi_2(\by', \bz')}}_{T_2} \},
    $$
    where we have defined
    $$
    \varphi_1(\by, \bz) \coloneqq (1-\beta)\by + \beta f(\bz), \quad \varphi_2(\by, \bz) = (1-\beta) \bz + \beta f((1-\beta)\by + \beta f(\bz)).
    $$
    Then showing \eqref{eq: claim} is equivalent to showing there exist $L_1, L_2 \in [0, 1)$ such that
    $$
    T_1 \leq L_1 \norm{(\by, \bz) - (\by', \bz')}_\infty, \quad T_2 \leq L_2 \norm{(\by, \bz) - (\by', \bz')}_\infty,
    $$
    and choosing $L = \max\left\{L_1, L_2\right\}$.

    For $T_1$, we estimate
    $$
    T_1 = \norm{(1-\beta)\by + \beta f(\bz) - (1-\beta)\by' - \beta f(\bz')} \leq \lvert 1-\beta\rvert\norm{\by - \by'} + \beta k \norm{\bz -\bz'}.
    $$
    Observing that
    $$
    \norm{\by - \by'} \leq \norm{(\by, \bz) - (\by', \bz')}_\infty, \quad \norm{\bz - \bz'} \leq \norm{(\by, \bz) - (\by', \bz')}_\infty,
    $$
    gives
    $$
    T_1 \leq L_1\norm{(\by, \bz) - (\by', \bz')}_\infty, \quad L_1 \coloneqq \lvert 1-\beta\rvert + \beta k.
    $$
    To find when $L_1 < 1$ we consider separately the cases $\beta =1$, $0 < \beta < 1$ and $\beta > 1$. When $\beta=1$, $L_1 = k < 1$ by assumption. When $0<\beta<1$ we require that $1-\beta+\beta k < 1$. This is equivalent to $k < 1$ which is true by assumption. When $\beta > 1$, we require that $\beta - 1 + \beta k < 1$. This is equivalent to $\beta < 2/(k+1)$.

    Similarly, for $T_2$ we estimate
    \begin{align*}
        T_2 &= \norm{(1-\beta)\bz + \beta f((1-\beta)\by + \beta f(\bz)) - (1-\beta)\bz' - \beta f((1-\beta)\by' + \beta f(\bz'))}, \\
        &\leq \lvert 1-\beta\rvert \norm{\bz - \bz'} + \beta k \norm{(1-\beta)(\by - \by') + \beta (f(\bz) - f(\bz'))}, \\
        &\leq \lvert 1-\beta \rvert \norm{\bz - \bz'} + \beta k (\lvert 1-\beta\rvert\norm{\by - \by'} + \beta k \norm{\bz - \bz'}), \\
        &\leq (\lvert 1-\beta \rvert + \beta k \lvert 1-\beta \rvert + \beta^2 k^2)\norm{(\by, \bz) - (\by', \bz')}_\infty,
    \end{align*}
    which gives
    $$
    T_2 \leq L_2\norm{(\by, \bz) - (\by', \bz')}_\infty, \quad L_2 \coloneqq \lvert 1-\beta \rvert + \beta k \lvert 1-\beta \rvert + \beta^2 k^2.
    $$
    To find when $L_2 < 1$, we again split into the cases $\beta=1$, $0< \beta < 1$ and $\beta > 1$. When $\beta = 1$, $L_2 = k^2 < 1$ by assumption. When $0 < \beta < 1$ we get the quadratic inequality
    $$
    1- \beta + \beta k (1-\beta)+ \beta^2 k^2 < 1.
    $$
    Rearranging and computing roots gives $\beta_- = -1/k$ and $\beta_+ = 0$. The inequality holds when either $\beta < \beta_-$ or $\beta > \beta_+$. Since $\beta > 0$ by definition, it holds for all $0 < \beta < 1.$ When $\beta > 1$, we get the quadratic inequality
    $$
    \beta - 1 + \beta k (\beta - 1) + \beta^2 k^2 < 1.
    $$
    Rearranging and computing roots gives $\beta_- = -1/k$ and $\beta_+ = 2/(k+1)$. The inequality holds for $\beta_- < \beta < \beta_+$ and, since we assume $\beta > 1$, this gives $1 < \beta < 2/(k+1)$.

    Thus, $\varphi$ is contractive for $0 < \beta < 2/(k+1)$, and setting $L = \max \{L_1, L_2\}$ gives \eqref{eq: claim}. Finally, we show that $L=L_1$. Observe that for $0 < \beta < 1$,
    $$
    L_1 \leq L_2 \iff 1-\beta + \beta k \leq 1-\beta + \beta k (1-\beta) + \beta^2 k^2 \iff k \geq 1,
    $$
    which is a contradiction, thus $L_1 > L_2$. Similarly, for $1 < \beta < 2/(k+1)$,
    $$
    L_1 \leq L_2 \iff \beta - 1 + \beta k \leq \beta - 1 + \beta k (\beta - 1) + \beta^2 k^2 \iff \beta \geq 2/(k+1),
    $$
    which is a contradiction, thus $L_1 > L_2$.

    To summarise, we have shown that for all $0 < \beta < 2/(k+1)$, $\varphi$ is a contraction under $\norm{\cdot}_\infty$ with contraction constant $L = \lvert 1-\beta \rvert + \beta k$.
\end{proof}

\begin{corollary}
    Assume that $f$ is a contraction with contraction constant $k \in [0, 1)$. Then, for all $0 < \beta < 2/(k+1)$, the reversible fixed point iteration $(\by_{n+1}, \bz_{n+1})$ converges linearly to the fixed point $(\by^*, \by^*)$ with speed $\lvert 1-\beta\rvert + \beta k$.
\end{corollary}

\begin{proof}
    Using Lemma \ref{lem:reversible-contractivity}, we write
    $$
    \norm{(\by_{n}, \bz_{n}) - (\by^*, \by^*)}_\infty = \norm{\varphi(\by_{n-1}, \bz_{n-1}) - \varphi(\by^*, \by^*)}_\infty \leq L \norm{(\by_{n-1}, \bz_{n-1}) - (\by^*, \by^*)}_\infty.
    $$
    Iterating this estimate gives
    $$
    \norm{(\by_{n}, \bz_{n}) - (\by^*, \by^*)}_\infty \leq L^{n}\norm{(\by_0, \bz_0) - (\by^*, \by^*)}_\infty,
    $$
    as claimed. Moreover, since the scheme is initialised with $\by_0 = \bz_0$, this implies linear convergence of each component (in the arbitrary norm $\norm{\cdot}$) to the fixed point $\bq^*$ of $f$, namely,
    $$
    \norm{\by_n - \bq^*} \leq L^n \norm{\by_0 - \bq^*}, \quad \text{and} \quad \norm{\bz_n - \bq^*} \leq L^n \norm{\bz_0 - \bq^*}.
    $$
\end{proof}

\section{Experiments}
\label{app:experiments}
\subsection{Language Modelling on Wikitext-103}
The model architecture is given by equation \eqref{eq:transformer}. We use an embedding size of $d=1024$, with $\mathbf{z}\in\mathbb{R}^d$ and $W_{QKV}\in \mathbb{R}^{3d\times d}$. The MLP has depth 2 with hidden size $4d$; the GELU \cite{hendrycks2016gaussian} activation function is used. The number of attention heads was 8. The token sequence length was 512 with a batch size of 32.

The reversible fixed point solver was applied to the equilibrium function for $N$ steps. The number of function evaluations is $2N$ and can be seen for each model in Table \ref{tab:perplexity-wikitext} and Table \ref{tab:nfe-wikitext}. We choose the tolerance to be $\epsilon = 10^{-3}$. The relaxation parameter for all models was $\beta=0.5$.

We use the Adam optimiser with a weight decay of $0.1$ and a peak learning rate of $10^{-3}$ \cite{loshchilov2019decoupled}. A cosine learning rate decay schedule was used with linear warm-up for $4000$ steps. Dropout was applied to the RevDEQ block with probability $0.1$. Training was performed for a maximum of 40 epochs.

In Table \ref{tab:perplexity-wikitext}, the Transformer-XL and DEQ results are from \cite{bai2021stabilizing}. In Table \ref{tab:nfe-wikitext}, the DEQ results with and without Jacobian regularisation are from \cite{bai2021stabilizing}.

\subsection{Image Classification on CIFAR-10}
\paragraph{Single-scale} The input image is first passed through an encoder: 3x3 convolutional layer (`same' padding) with 64 output channels, followed by a GroupNorm layer. Next we apply the RevDEQ block: the first 3x3 convolution (`same' padding) increases the channels dimension to 128 (width expansion equal to 2), the second convolution decreases the channel dimension back to 64. Finally, we average pool to a 4x4 image and the linear classification layer is applied to compute the logits.

The reversible fixed point solver was applied to the equilibrium function for $N$ steps. The number of function evaluations is $2N$ and can be seen in Table \ref{tab:cifar10}. We choose the tolerance to be $\epsilon = 10^{-6}$. The relaxation parameter was taken as $\beta=0.8$.

We use the Adam optimiser with a weight decay of $10^{-4}$ and an initial learning rate of $10^{-3}$ \cite{loshchilov2019decoupled}. The learning rate schedule is a decay on plateau schedule; we half the learning rate if the training loss does not increase for 10 steps (accumulated over the past 200 steps). Training is performed for a maximum of $100$ epochs or if the learning rate drops below $10^{-6}$. The batch size was 64.

\paragraph{Multi-scale} The input image is first passed through an encoder (same as single-scale) that increases the channel dimension to $C_1$, where $C_1$ is the channel dimension of the first scale. This is followed by a BatchNorm layer and ReLU. Next we apply each scale of the multi-scale model; the parameters at each scale are the channel dimension $C_i$ and the width expansion inside the equilibrium function. The reversible fixed point parameters are the maximum number of steps $N_i$ and the tolerance $\epsilon_i$, which is fixed to $\epsilon_i=10^{-6}$ for all scales $i$. The relaxation parameter for all scales was $\beta=0.8$. The parameters for each model are given in Table \ref{tab:app-cifar10}.

After each RevDEQ block we apply downsampling; for scales $i=1, 2, 3$ we use the pre-activation downsampling block from \cite{he2016identity} with BatchNorm. At the final scale, $i=4$, we downsample using a global average pooling operation. Finally, the linear classification layer is applied to compute the logits.

The optimisation setup is the same as the single-scale description but with an initial learning rate of $3 \times 10^{-4}$.

In Table \ref{tab:cifar10}, the monDEQ results are from \cite{winston2020monotone}; the pcDEQ results are from \cite{gabor2024positive}; the DEQ, MDEQ and explicit ResNet results are from \cite{bai2020multiscale}. The test accuracy is shown as \emph{mean $\pm$ std} over 3 repeats.

\begin{table}
    \centering
    \renewcommand{\arraystretch}{1.3}
    \begin{tabular}{c c c c}
         \toprule
         Model & Channels & Width Expansion & Solver Steps \\
         \midrule
         RevDEQ (170K) & [32, 32, 32, 32] & [1, 2, 2, 1] & [1, 4, 4, 1] \\
         RevDEQ (5M) & [64, 128, 128, 256] & [2, 4, 4, 2] & [1, 4, 4, 1] \\
         RevDEQ (10M) & [128, 256, 256, 128] & [1, 3, 3, 1] & [4, 4, 4, 4] \\
         \bottomrule
    \end{tabular}
    \vspace{1ex}
    \caption{Multi-scale RevDEQ hyper-parameters.}
    \label{tab:app-cifar10}
\end{table}

\end{document}